\documentclass[conference]{IEEEtran}

\usepackage{cite}
\usepackage{amsmath,amssymb,amsfonts}
\usepackage{amsthm}
\usepackage{algorithmic}
\usepackage{graphicx}
\usepackage{textcomp}
\usepackage{xcolor}
\usepackage{subfig}
\newtheorem{prop}{Proposition}
\newtheorem{definition}{Definition}
\def\BibTeX{{\rm B\kern-.05em{\sc i\kern-.025em b}\kern-.08em
    T\kern-.1667em\lower.7ex\hbox{E}\kern-.125emX}}
\begin{document}

\title{Towards Consistent Predictive Confidence through Fitted Ensembles\\}

\author{\IEEEauthorblockN{Navid Kardan}
\IEEEauthorblockA{\textit{Department of Computer Science} \\
\textit{University of Central Florida}\\
Orlando, USA \\
nkardan@cs.ucf.edu}
\and
\IEEEauthorblockN{Ankit Sharma}
\IEEEauthorblockA{\textit{Department of Computer Science} \\
\textit{University of Central Florida}\\
Orlando, USA \\
ankit.sharma285@knights.ucf.edu}
\and
\IEEEauthorblockN{Kenneth O. Stanley}
\IEEEauthorblockA{\textit{Department of Computer Science} \\
\textit{University of Central Florida}\\
Orlando, USA \\
0000-0003-4503-0839 }
}

\maketitle

\begin{abstract}
Deep neural networks are behind many of the recent successes in machine learning applications. However, these models can produce overconfident decisions while encountering out-of-distribution (OOD) examples or making a wrong prediction. This inconsistent predictive confidence limits the integration of independently-trained learning models into a larger system. This paper introduces \emph{separable concept learning} framework to realistically measure the performance of classifiers in presence of OOD examples. In this setup, several instances of a classifier are trained on different parts of a partition of the set of classes. Later, the performance of the combination of these models is evaluated on a separate test set. Unlike current OOD detection techniques, this framework does not require auxiliary OOD datasets and does not separate classification from detection performance.  
Furthermore, we present a new strong baseline for more consistent predictive confidence in deep models, called \emph{fitted ensembles}, where overconfident predictions are rectified by transformed versions of the original classification task. Fitted ensembles can naturally detect OOD examples without requiring auxiliary data by observing contradicting predictions among its components. Experiments on MNIST, SVHN, CIFAR-10/100, and ImageNet show fitted ensemble significantly outperform conventional ensembles on OOD examples and are possible to scale.  
\end{abstract}

\section{Introduction}
Classification is a fundamental problem of machine learning \cite{breiman2017classification,DudaHart2nd,dietterich1998approximate,duda1973pattern,cover1967nearest,fisher1936,hastie,fisher1938,vapnik}, where, the goal is to assign an observation (instance) to a set of classes.

An important subclass of classifiers is \emph{probabilistic classifiers}, which predict a probability distribution over a set of classes rather than just the most likely class. The classifier's decision can then be the class with the highest probability. More formally, while a classifier is a mathematical function $h:X\mapsto Y$ that assigns an instance $x \in X$ to a class $y \in Y$, a probabilistic classifier outputs the conditional probability $p(Y|X)$. This approach is appealing because, for example, these probability values can be considered as confidence values that can be used to reject uncertain observations.

Neural networks are naturally probabilistic classifiers. They are designed to assign probability values to different classes so that the sum of these values is one. However, a possible caveat in this approach is vulnerability to \emph{out-of-distribution} examples (overgeneralization) or in-distribution examples that the classifier is wrongly confident about (confident misclassification). We combine these two problems, i.e.\ overgeneralization, and confident misclassification, and refer to them as \emph{inconsistent prediction confidence}.

Fitted ensembles can naturally detect OOD examples without requiring auxiliary data by observing contradicting predictions among its components. Experiments on MNIST, SVHN, CIFAR-10/100, and ImageNet show fitted ensemble significantly outperform conventional ensembles on OOD examples and are possible to scale.  

In short, the contributions of this paper are:
\begin{itemize}
\item Introduction of \emph{separable concept learning} framework to streamline measurement of classifiers' performance in presence of relevant OOD examples. This framework does not rely on auxiliary OOD datasets and combines classification and detection performance.  
\item We propose a new ensembling strategy, called \emph{fitted ensembles}, to achieve more consistent predictive confidence in deep models.
\end{itemize}

\section{Related Works}
Inconsistent prediction confidence is related to model calibration\cite{guo2017calibration,MM2020}, out-of-distribution detection\cite{lecun1989backpropagation,our,hendrycks2016baseline,liang2017enhancing,lee2018simple,devries2018learning,hendrycks2018deep,ApoorvECCV,sastry2020detecting,chen2020robust,Energy2020}, and uncertainty estimation\cite{galuncertainty,Ovadia,posterior2020,Ustatistics2020}. Model calibration in neural networks tries to calibrate their predictive probabilities so that they match their accuracy. A common technique for calibration is temperature scaling \cite{guo2017calibration}, where outputs of the neural network are divided by a scalar, called temperature, before passing to the softmax layer.

By contrast, for consistent predictive confidence, the goal in regards to in-distribution (familiar) examples, is to introduce techniques that minimize predictive confidence for misclassified examples. The discrepancy between the expected confidence of the model and its classification accuracy can then possibly be corrected by a model calibration method. 

In the out-of-distribution detection problem, a threshold over the confidence values for each prediction teases out out-of-distribution examples. Like model calibration, this is also done as a postprocessing procedure. LeCun et.~al.~\cite{lecun1989backpropagation} utilize the difference of the first and the second highest class probabilities as such confidence measure. Hendrycks et.~al.~\cite{hendrycks2016baseline} showed that simply applying the maximum softmax probability is a good baseline for out-of-distribution detection in deep NNs.

DeVries et.~al.~\cite{devries2018learning} train a NN with an additional output unit responsible for confidence estimation. Liang et.~al.~\cite{liang2017enhancing} propose to detect out-of-distribution examples in a pretrained model based on temperature scaling and input preprocessing, where small perturbations are added to an input to boost prediction confidence before inference. Their method assumes access to a representative set for out-of-distribution data. As an example of a metric-based out-of-distribution detection method, in \cite{lee2018simple} each class is modeled by a Gaussian distribution with a joint covariance matrix. Later Mahalanobis distance is applied to detect out-of-distribution and adversarial examples. Finally, Hendrycks et.~al.~\cite{hendrycks2018deep} propose to expose the model with auxiliary data during training to minimize the model's confidence (a uniform output) on the extra examples. Their work can be thought of as an application of universum examples \cite{zhang2017universum} for out-of-distribution detection.

Most of these methods model out-of-distribution data by an auxiliary dataset, which effectively reduces out-of-distribution detection to a binary classification problem. By contrast, in this paper, we model out-of-distribution data as a complement of in-distribution data and assume no access to auxiliary data.  

Confidence values can be interpreted as the inverse of uncertainty values, which are not restricted to the classification problem. In the context of deep models, \cite{galuncertainty} derive a connection between Bayesian inference in Gaussian processes and neural networks that apply dropout which enables them to model uncertainty. Lakshminarayanan et.~al.~\cite{lakshminarayanan2017simple} estimate uncertainty statistics from an ensemble of classifiers and subsequently detect out-of-distribution examples. Unlike most methods, they showed their technique is scalable by applying it to the ImageNet-2012 dataset. Our proposed fitted ensemble improves upon this approach by adding a new dimension to regular ensemble classifiers that rectify their predictive confidence by reformulating the original classification problem and lowering conflict between different components.

\section{Separable Concept Learning}

Recall that inductive biases are the assumptions that the model have about the unseen instances. For example, some of the common biases in machine learning are:
\begin{itemize}
\item minimum cross-validation error,
\item maximum margin,
\item nearest neighbors, and
\item minimum description length (Occam's razor).
\end{itemize}
To mitigate overgeneralization, in this section we introduce another bias, called \emph{separable concept learning} (SCL). In SCL we assume a combined classifier should work well if multiple instances of it are trained on disjoint subsets of classes (a mathematical partition of the set of classes).

Here we assume evaluation of a classifier consists of two components: model performance on instance $\mathbf{x}$ when (1) $\mathbf{x} \sim D^*(\mathbf{x}|y\in Y)$, and (2) when $\mathbf{x} \sim D^*(\mathbf{x}|y\in Y^c)$, where $D^*$ is the data distribution over similar data as in the original data distribution $D$ (including $D$), and $Y^c$ denotes the complement of set of classes $Y$. We call the second component the \emph{inhibition ability} of the model because it must inhibit its tendency to activate any output class. While the first component can be evaluated by any classification metric, e.g.\ classification accuracy, to also account for the inhibition ability we apply separable concept learning (SCL).

The advantage of SCL over previous evaluation procedures applied in out-of-distribution detection tasks is that it does not rely on any threshold and measures the rejection ability of the model implicitly without requiring auxiliary data to simulate $D^*$. This is an important property because, in light of no free lunch \cite{wolpert1997no}, current evaluation techniques are highly dependent on the choice of this auxiliary set. The SCL provides a more natural and arguably useful way to measure the inhibition ability of learning models.

Intuitively, a set of classifiers should be able to learn different concepts separately from each other because they are aware of the limits of their knowledge. As a result, these models can learn new non-overlapping concepts simply by accepting new output nodes from a separately trained model. In other words, it is possible to train several classifiers over different concepts and merge them to get a single model over the union of concepts. In this way, SCL is simulating the way learning models can handle new emerging relevant concepts and work together.

A closely related subject to SCL is \emph{incremental concept learning} (ICL), where concepts are presented to the learning model incrementally \cite{ourarxiv}. The SCL can be conceived as the extreme case of ICL where after learning a few concepts their training instances diminish from the memory. Interestingly, ICL and SCL are more reminiscent of human learning than the current full concept learning approach. Furthermore, in most realistic applications the size of $Y$ is unknown at the beginning or could grow over time, which makes ICL a practically important subject in supervised learning.

Formally, to evaluate the inhibition ability of a model, let $\mathrm{concat}$ be the concatenation operator that merges class scores according to a predefined order over the set of classes $Y$ and then normalizes them, $l(.,.)$ be a loss function, and $\rho$ be a partition of $Y$ such that every set in $\rho$ (part of the partition) is at least of size $2$ (contains at least two classes from $Y$). The \emph{empirical separable risk} of a classifier $g$ with respect to $\rho$ is then defined as:
\begin{equation} \label{metric}
r_{\rho}(g)=\frac{1}{n}\sum_{i=1}^{n}l(\mathrm{concat}_{\forall \rho_k \in \rho}(g_k(\mathbf{x_i})),y_i),
\end{equation}
where $g_k$ is an instance of $g$ trained on $\rho_k$ (for convenience assume members of $\rho$ are indexed by $k$), i.e.\ a learning model trained on the subset of training set samples whose classes are only included in $\rho_k$. 

\begin{definition}{\textbf{SCL Accuracy}.}
The SCL accuracy is calculated according to equation~\ref{metric}, where the predicted class is chosen according to the optimal decision rule\cite{bishop2006pattern} that predicts the class with the highest probability, and loss function is the indicator function.
\end{definition}
Note that in the case of $\rho=\{\{Y\}\}$, SCL accuracy will reduce to the classification accuracy measure. More generally: 

\begin{prop}\label{prop1}

The SCL accuracy is a lower bound for the average classification accuracy of the component classifiers $g_k$ when the test set is partitioned according to $\rho$ and each $g_k$ is evaluated on its corresponding set, i.e.\  
\begin{equation}\label{bound}
r_{\rho}(g)\leq \frac{1}{n}\sum_{i=1}^{n}l(g_{k:y_i\in\rho_k}(\mathbf{x_i}),y_i).    
\end{equation}

\end{prop}
\begin{proof}
The final normalization of the $\mathrm{concat}$ operator does not change the order of class probabilities, thus it does not affect the final classification decision and can be ignored in our analysis. To assign test example $x_i$ to its actual label $y_i$, its corresponding probability $p(y_i|x_i)$ should be maximum. Because $p(y_i|x_i)$ is assigned by the $g_k$ that is trained on this class, i.e.\ $y_i\in\rho_k$, and concatenation cannot increase this value, therefore $\mathrm{concat}$ can only either decrease the accuracy or preserve it.
\end{proof}

While SCL accuracy is a simple measure to compute classification accuracy in presence of out-of-distribution examples, equation~\ref{bound} provides a way to measure the exact amount of accuracy reduction in SCL tasks that occur due to the existence of out-of-distribution data.
Abstractly, while comparing two learning models, $M_1$ is strictly better than $M_2$ if $\forall \rho: r_{\rho}(M_1)\leq r_{\rho}(M_2)$. However, the possible ways of partitioning $Y$ grows exponentially with the size of $Y$. Thus in practice, an approximation is used by sampling from all the possible partitions.

Informally, to apply equation \ref{metric}, one can train several components of a SCL model on different subsets of the set of classes and measure the performance of the whole (merged) model over the test set. This approach provides a simple procedure to evaluate the inhibition ability of different learning models. As we will later see, fitted ensembles can significantly boost performance over conventional NNs on SCL tasks.

\section{Fitted Ensembles}
Lakshminarayanan ~et.~al.~\cite{lakshminarayanan2017simple} showed that applying conventional ensembles is more effective than MC-dropout for detecting out-of-distribution examples. They showed this approach scales well to higher-dimensional spaces while at the same time improving the overall classification accuracy. In the out-of-distribution experiment section we will show that:
\begin{itemize}
    \item conventional ensemble techniques fail to reliably detect many out-of-distribution examples on popular datasets, and
    \item additionally, ensemble methods suffer from overconfident predictions.
\end{itemize}  

Our proposed model addresses both of these drawbacks and adjusts the class probabilities directly. Moreover, fitted ensembles do not sacrifice classification accuracy to reduce overgeneralization.

The main idea is to diversify the induced feature space by changing the original classification task. This is achieved by deriving new classification problems from the original problem. We extract new classification problems by combining original classes into larger superclasses to form new \textbf{superclass spaces}. A set of one or more such superclass spaces is referred to as a \textbf{sequel}. A fitted ensemble is a non-empty set of sequels.

To provide more clarity on sequels, consider a four-class classification problem with output space $Y=\{A,B,C,D\}$. We construct a sequel $H$ consisting of two superclass spaces $\{h^0,h^1\}$ in the following manner, $H=\{h^0,h^1\},$ where,
\begin{center}
    $h^0=\{h^0_0,h^0_1\},$
    $h^1=\{h^1_0,h^1_1\},$ 
\end{center}
\begin{center}
    $h^0_0=\{A,D\},$
    $h^0_1=\{B,C\},$ 
\end{center}
\begin{center}
    $h^1_0=\{A,B\},$ 
    $h^1_1=\{C,D\}.$
\end{center}

Here, each superclass ($h^0_0, h^0_1, h^1_0,$ and $h^1_1$) includes two classes from the original problem. In practice, we simplify sequels by assuming that each superclass space $h^j$ is a partition of set of original classes. A sequel consists of only those superclass spaces that can solve the original classification problem. In the sequel $H$ discussed earlier, superclass space $h^0$ can only inform that an example belongs to either $\{A,D\}$ or $\{B,C\}$, but not to one particular class. With the addition of space $h^1$, we can deduce to which original class the example belongs and solve the original classification problem. Essentially, a sequel is a conceptual means to construct better superclass spaces.

More formally, assume output space $Y$ consists of $n$ categories $y_i$, where $i\in\{1 .. n\}$. We map $Y$ into $r$ new superclass spaces $h^j$ each with $|h^j|\leq n$ elements. In this new space, $\forall ij: h^j_i\subset Y$ and $\forall i\neq k: h^j_i\cap h^j_k=\emptyset$. As mentioned earlier, each superclass space sets up a new classification problem. And, for each such space we have its corresponding member-classifier. To put it simply, a fitted ensemble model consists of \emph{r} member-classifiers \emph{g$^{j}$} \emph{j} $\in$ \{1,...,\emph{r}\}, where \emph{g$^{j}$} is trained on $h^j$.  
Each superclass space $h^j$ defines an upper bound on the predictive confidence of the model. Given an example $x$, a class probability cannot be greater than the probability of any superclasses it belongs to, i.e. $\forall y_i\in h^j_k:p(y_i|x)\leq p(h^j_k|x)$. During inference, each \emph{g$^{j}$} corrects the current confidence about class probabilities by adjusting the estimated probability values for each class in a fitted ensemble. In particular, the class probabilities should satisfy all the inequalities in the form of $\widehat{p(y_i|x)}\leq g^j_k(x)$, where $y_i\in h^j_k$, and $\widehat{p(y_i|x)}$ denotes the estimation of $p(y_i|x)$ by the model. In this manner, a fitted ensemble can indirectly make inferences about OOD examples by deducing that the class probabilities are closer to zero for unknown examples. Algorithms 1 summarizes the necessary steps for constructing and training a fitted ensemble. The procedure for rectifying class probabilities during inference phase is outlined in algorithm 2.  

\begin{table}[t]
\centering
\begin{tabular}{l}
\hline
\multicolumn{1}{c}{\textbf{\begin{tabular}[c]{@{}c@{}}Algorithm 1: Fitted Ensemble Construction\end{tabular}}} \\ \hline
\begin{tabular}[c]{@{}l@{}} \textbf{Input} 
Dataset $\mathbf{D}$ with original classes $Y$          \\
number of sequels $m$   \\
size of superclass spaces for each sequel $\mathbf{s_i}$ \\\end{tabular}, \\
set of superclass members $H$\\
set of member-classifiers $G$\\
1. {\bfseries set} $H$ and $G$ to $\emptyset$ \\
2. \bfseries{for} $i=1$ {\bfseries to} $m$          \\
3. \hspace{15pt}\bfseries{for} $j=1$ {\bfseries to} $\mathbf{s}_i$          \\
4. \hspace{30pt} randomly partition $Y$ into $|Y|/\mathbf{s}_i$ superclasses\\   
5. \hspace{30pt} append this partition to $H$ \\
6. \hspace{30pt} create a new dataset with the new superclass labels \\ 
7. \hspace{30pt} train a member-classifier $g$ on the new dataset \\ 
8. \hspace{30pt} append $g$ to $G$ \\
9. \hspace{15pt}\bfseries{end for}      \\
10. \bfseries{end for}     \\
13. Return $G$ and $H$    \\ \hline
\end{tabular}
\end{table}

\begin{table}[t]
\centering
\begin{tabular}{l}
\hline
\multicolumn{1}{c}{\textbf{\begin{tabular}[c]{@{}c@{}}Algorithm 2: Probability rectification \\ in Fitted Ensemble\end{tabular}}} \\ \hline
\begin{tabular}[c]{@{}l@{}} \textbf{Input} test example $\mathbf{x}$\\ member-classifiers $G$\\ set of superclass members $H$\\ total number of classes $n$\end{tabular} \\
1. \bfseries{for} $i=1$ {\bfseries to} $n$                 \\
2. \hspace{15pt}Initialize $p(y_i|\mathbf{x}) = 1$ \\
3. \bfseries{end for}   \\
4. \bfseries{for} $i=1$ {\bfseries to} $|G|$          \\
5. \hspace{15pt}{\bfseries set} $G^i$ to $\mathbf{g}$ \\
6. \hspace{15pt}calculate probability vector $\mathbf{g}(\mathbf{x})$                \\
7. \hspace{15pt}\bfseries{for} $j=1$ {\bfseries to} $|\mathbf{g}(\mathbf{x})|$   \\
8. \hspace{30pt} {\bfseries for} all class indices $s$ in superclass $H^i_j$                    \\
9. \hspace{45pt}$p(y_s|\mathbf{x})$ = $min(p(y_s|\mathbf{x}), g_j(\mathbf{x}))$         \\
10. \hspace{30pt}\bfseries{end for}                                       \\
11. \hspace{15pt}\bfseries{end for}                                   \\
12. \bfseries{end for}                           \\
13. Return class probability vector $\mathbf{p}(\mathbf{y}|\mathbf{x})$.    \\ \hline
\end{tabular}
\label{calib_alg}
\end{table}

\section{FITTED ENSEMBLE EXPERIMENTS}

In this section the MNIST, CIFAR-10/100, SVHN, and ImageNet-2012 datasets support a variety of experiments that assess the performance of fitted ensembles. The general setup is to train a model on a dataset that is representative of the true data distribution, and then evaluate the confidence of the model predictions on both unseen and in-distribution classes. Following \cite{lakshminarayanan2017simple,sastry2020detecting} and unlike many new out-of-distribution modeling methods, we assume the model has no access to any out-of-distribution data, neither directly as in \cite{liang2017enhancing,ApoorvECCV,lee2018simple}, nor indirectly as in \cite{hendrycks2016baseline,hendrycks2018deep,chen2020robust}. Hence, it is not comparable to these methods. Therefore, we focus our comparison to conventional ensembles as an established strong method \cite{Ovadia} to address out-of-distribution problem.

Following \cite{lakshminarayanan2017simple}, we plot histograms of different confidence values, where a good model should be less confident about out-of-distribution examples. Moreover, a model's confidence should decrease for in-distribution examples that it misclassifies. The results suggest fitted ensembles are indeed successful at both of these tasks and significantly improve on conventional ensembles in recognizing examples from unfamiliar classes.  

For consistency following \cite{lakshminarayanan2017simple}, all the models applied in these experiments follow a small VGG-like \cite{simonyan2014very}
architecture for each classification/member problem, except for ImageNet experiment in which we use DenseNet \cite{huang2017densely} 
architecture to show the results are generalizable to different architectures. 

\subsection{The Set of Sequels}
Except for in the ImageNet experiment, each fitted ensemble includes four extra superclass spaces (two sequels) in addition to the regular classification problem. Two of these spaces are created by combining every two consecutive classes from the original dataset, starting from $0$ and $1$, respectively, to form superclasses. For datasets with $10$ classes that means $H=\{h^0, h^1\}$, where $h^0=\{ h^0_i=\{y_{2i},y_{1+2i}\}:i\in \{0..4\}\}$, and $h^1=\{ h^1_i=\{y_{1+2i},y_{(2+2i)\textrm{mod} 10}\}:i\in \{0..4\}\}$. 

The next two superclass spaces are created similarly, except that we combine every other class. These four superclass spaces add $20$ constraints to adjust the primary probability values (four for each class). In the ImageNet experiment, in addition to the regular classification problem, the fitted ensemble includes two extra superclass spaces (one sequel), each of which is constructed by combining every two consecutive classes (according to the PyTorch\footnote{https://pytorch.org/} framework's default class ordering) from the original dataset.

To confirm that adding superclass spaces does not negatively affect classification performance, first we trained an ensemble of five CNNs and compared its performance to an ensemble of five fitted ensembles.
Note that each fitted ensemble contains the five member networks as just explained.  While the aggregate of fitted ensembles takes more space, the question is whether that additional space affords superior detection of out-of-distribution examples without loss of in-distribution accuracy.

For the aggregate of fitted ensembles, we combined the results of the five different fitted ensembles at the superclass spaces level. That is, Algorithm~2 is applied after combining outputs of all corresponding member-classifiers. The combination is always a weighted average with uniform weights.

Table~\ref{acc-table} 
depicts the classification accuracy of each of the two ensembles when run once on a variety of datasets. The table shows that the aggregate of fitted ensembles successfully maintains (or improves) the classification performance of its base classifiers after adding sequels. Note that we utilize small network architectures (wider architectures can improve accuracy on CIFAR-100).

\begin{table}[t]
\caption{Classification accuracy of ensembles of five CNNs and fitted ensembles on various data sets.}
\label{acc-table}
\vskip 0.15in
\begin{center}
\begin{small}
\begin{sc}
\begin{tabular}{lcccr}
\hline
Data set & CNNs & Fitted Ensembles  \\
\hline
MNIST    & 99.65& 99.70\\
SVHN    & 96.95& 97.10\\
CIFAR-10    & 93.1& 93.4 \\
CIFAR-100    & 69.15& 69.85 \\

\hline
\end{tabular}
\end{sc}
\end{small}
\end{center}
\vskip -0.1in
\end{table}

\subsection{Confidence Consistency Experiments}
In the rest of the experiments of this section, we evaluate the out-of-distribution recognition and overconfidence reduction ability of fitted ensemble. In these experiments, we apply an ensemble of five fitted ensembles (each with five superclass spaces) and compare its performance with an ensemble of $50$ CNNs, which is twice as many neural networks as the aggregate of fitted ensembles. \emph{A large number of neural networks in regular ensembles ($50$) rules out any further improvement of confidence values because of conventional ensembling.}
Because \cite{lakshminarayanan2017simple} have already shown that ensembles significantly outperform MC-dropout for uncertainty estimation we did not include MC-dropout results in these experiments.

All the experiments in this section include a VGG-like architecture for neural networks. This architecture includes nine ($3\times 3$) convolutional layers each with $100$ output feature maps, followed by two fully-connected layers (of size $1,000$) and a softmax layer. All layers are followed by ReLU non-linearity and an additional batch-normalization layer. Furthermore, there is a max-pooling layer, after three consecutive convolutional layers. 
The weight initialization follows \cite{he2015delving} for convolutional layers and \cite{glorot2010understanding} for linear layers.

The models are trained by Adam optimizer \cite{kingma2014adam} with a batch size of $40$ in all our experiments. The learning rate always starts at $0.01$ and is decayed by multiplying by $0.4$ after every $10$ epoch. All the models are trained for $52$ epochs. Even though the hyperparameters are not optimized for any specific dataset, we found our setup generally robust on all the four datasets studied in this work.

The loss function is always cross-entropy and no dropout or any other regularization, except for batch-normalization, was applied. For preprocessing all the datasets were normalized.  

A light data augmentation was performed on each dataset. For MNIST and SVHN dataset, the images are first zero padded to obtain $34\times 34$ resolution images and then randomly cropped to get to $32\times32$ images.
The CIFAR-10/100 datasets are first padded to $36\times 36$ images and then randomly cropped to obtain $32\times 32$ images. The training images in these datasets are also randomly horizontally flipped.

\subsubsection{MNIST Experiment}
The MNIST dataset consists of gray-scale images of digits and is separated into $60$,$000$ training and $10$,$000$ testing images. First an ensemble of $50$ regular CNNs and five fitted ensembles ($25$ CNNs) were trained on MNIST.

For out-of-distribution data we applied notMNIST\footnote{Available at http://yaroslavvb.blogspot.co.uk/2011/09/notmnist-dataset.html}, a dataset similar to MNIST but consisting of alphabet letters.

Due to space limitations, we exclude the graph of this experiment. Nevertheless, we observe a substantial decrease of highly-confident out-of-distribution examples in the aggregate of fitted ensembles (from about 34\% down to about 15.5\%), which is evidence of the effectiveness of our approach. Furthermore, there is a large amount of highly-confident examples from the notMNIST dataset in the CNN ensemble that showcases the ubiquity of the overgeneralization (all $50$ CNN models are over 95\% confident about 34\% of out-of-distribution examples). 

It is also noteworthy that a considerable amount (about 15\%) of out-of-distribution examples have been correctly rejected by the fitted ensemble aggregate ensemble, whereas the CNN ensemble is never totally unsure (considering that random probability is 10\%) about any out-of-distribution example.

\subsubsection{SVHN Experiments}
This next experiment evaluates fitted ensembles on the SVHN dataset of colorful digit images. Higher noise in examples and greater input space makes this dataset more challenging than MNIST. The test set of the CIFAR-10/100 dataset serves as the out-of-distribution data. 

\begin{figure*}[!ht]
\vskip 0.2in
\centering
\subfloat[50 CNNs]{

\includegraphics[height=1.8in]{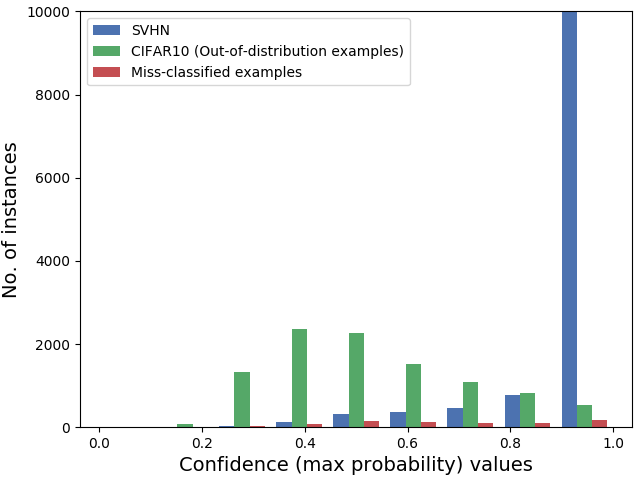}}
\qquad
\subfloat[five fitted ensembles]{

\includegraphics[height=1.8in]{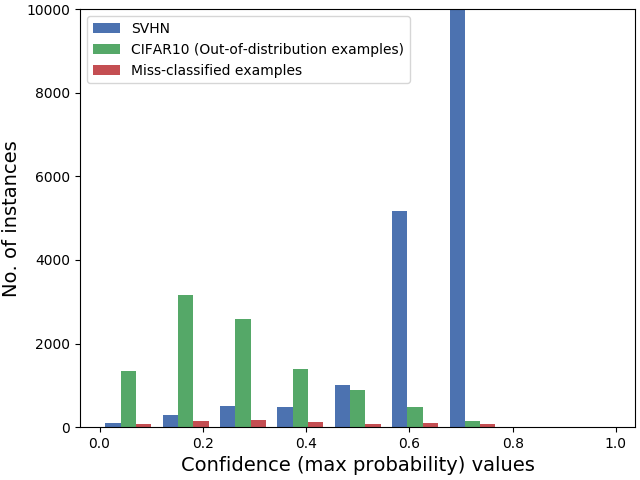}}%
\caption{\textbf{Distribution of confidence values for SVHN (in-distribution) and CIFAR-10 (out-of-distribution) examples.} (a) the result from $50$ CNNs of the ensemble show the limit of ensembling for improving confidence values. (b) an aggregate of five fitted ensembles significantly outperform the regular ensemble. Note that some of the bars go above the $10$,$000$ upper bound of the graph because there are $26$,$032$ test examples in SVHN. }
\label{svhn_sag}
\vskip -0.2in
\end{figure*}

Considering the CIFAR-10 test set as out-of-distribution data, histograms of confidence values are depicted in figure~\ref{svhn_sag} for an ensemble of (a) $50$ CNNs (b) five fitted ensembles. Note that there are $26$,$032$ test examples in SVHN (some of the bars go above $10$,$000$ examples). The plots illustrate that both ensembles are successful in this task. However, the aggregate of fitted ensembles rejects more out-of-distribution examples and adjusts the confidence values more accurately.

In contrast, the ensemble of CNNs assigns high-confidence values to many out-of-distribution examples and is very confident about a significant proportion of misclassified examples. 
A similar result is observed when applying CIFAR-100 as an out-of-distribution dataset (graphs are excluded due to space constraints). The results reaffirm that the aggregate of fitted ensembles performs better in out-of-distribution recognition.

\subsubsection{CIFAR-10 and CIFAR-100 Experiments}
CIFAR-10/100 are datasets of tiny images of $10$ and $100$ different types of animals and moving objects, respectively. The datasets are considerably more challenging than MNIST and SVHN because of the greater variety of images. We observe similar trends when we compare an ensemble of $50$ regular CNNs versus five fitted ensembles trained on these datasets and treat SVHN and the remaining CIFAR-10/100 as out-of-distribution data. However, due to space constraints, we only include the result for models trained on CIFAR-10 compared against CIFAR-100 data as unseen classes in Figure~\ref{cifar10_sag}. 

Overall, the results exhibit again the superior performance of fitted ensembles in assigning low confidence values or rejecting out-of-distribution data on out-of-distribution data. Furthermore, fitted ensembles assign lower confidence values to misclassified examples.

\begin{figure*}[!ht]
\vskip 0.2in
\centering
\subfloat[50 CNNs]{%
\includegraphics[height=1.8in]{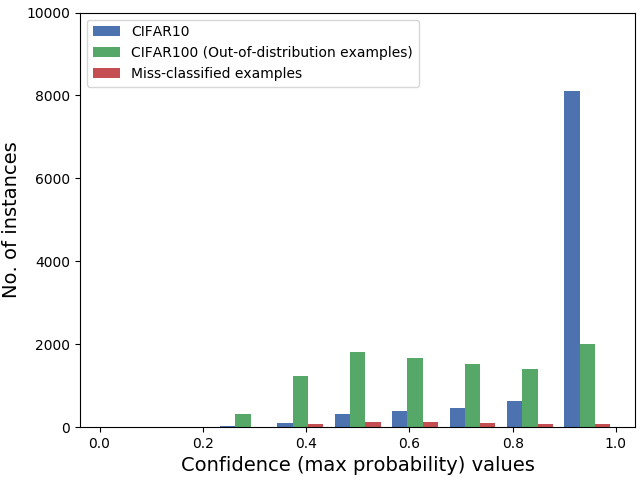}}
\qquad
\subfloat[five fitted ensembles]{%
\includegraphics[height=1.8in]{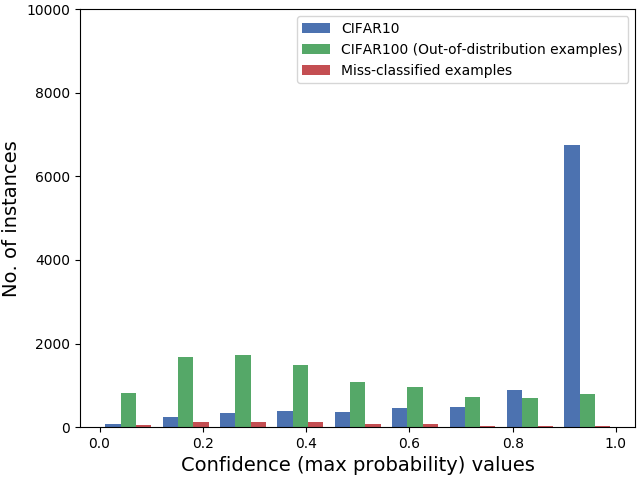}}%
\caption{\textbf{Distribution of confidence values for CIFAR-10 (in-distribution) and CIFAR-100 (out-of-distribution) examples for an ensemble of (a) $50$ CNNs and (b) five fitted ensembles.}
}
\label{cifar10_sag}
\vskip 0.2in
\end{figure*}

\subsection{SCL Experiments}
To demonstrate the performance of fitted ensemble on SCL we evaluated the accuracy of an ensemble of three fitted ensembles versus an ensemble of three conventional CNNs. We applied the same datasets, neural network architecture, and training procedure from the past experiments for all neural networks. 
However, we applied four superclass spaces (including the original classification problem) for each fitted ensemble. The set of sequels for fitted ensembles are changed to conform with five total number of classes. 

In addition to regular classification problem, we added three extra constraint-problems for MNIST, SVHN and CIFAR-10 datasets, as follows:
$H=\{h^0=\{h^0_0=\{y_0\},h^0_1=\{y_1\},h^0_2=\{y_2\},h^0_3=\{y_3,y_4\}\},h^1=\{h^1_0=\{y_0,y_1\},h^1_1=\{y_2\},h^1_2=\{y_3\},h^1_3=\{y_4\}\},h^2=\{h^2_0=\{y_0\},h^2_1=\{y_1,y_2\},h^2_2=\{y_3\},h^2_3=\{y_4\}\}\}$. Finally, for CIFAR-100 with $50$ classes in the SCL task, we applied the same four superclass spaces as in the out-of-distribution examples experiments.

For SCL, each training set was split into two sub-training sets. The split ensures that the first sub-training set includes all the examples from the first half of classes from the original training set, and the second sub-training set includes the rest. This split is arbitrary because the set of classes are independent in each dataset.

After training two classifiers on each dataset, the classifiers are merged and applied to classify the test data. Table~\ref{dcl-table} summarizes the classification accuracy results with this procedure on different datasets. Each experiment is repeated five times except for CIFAR-100, which is a single run. The results show that the aggregate of fitted ensembles significantly ($p<0.01$ based on the Wilcoxon test) outperforms conventional ensembles on each SCL task with multiple runs (MNIST, SVHN, and CIFAR-10).

\begin{table}[t]
\caption{SCL accuracy of ensembles of CNNs and fitted ensembles on various data sets (mean and standard deviation is from five runs). Aggregate of fitted ensembles lead to significantly better classification accuracy in the presence of unseen classes.}
\label{dcl-table}
\vskip 0.15in
\begin{center}
\begin{small}
\begin{sc}
\begin{tabular}{lcccr}
\hline
Data set & CNNs (std) & fitted ensembles (std) &  \\
\hline
MNIST    & 98.27 (0.17)& \textbf{98.70 (0.08)}\\
SVHN    & 93.11 (0.09)& \textbf{94.01 (0.08)}\\
CIFAR10    & 82.39 (0.36)& \textbf{83.45 (0.19)} \\
CIFAR100    & 60.80 & 62.80 \\

\hline
\end{tabular}
\end{sc}
\end{small}
\end{center}
\vskip -0.1in
\end{table}

\subsection{ImageNet Experiments}
ImageNet-2012 dataset consists of about $1.2$ million images with $1$,$000$ classes that were originally used for a vision competition in 2012. All the images in the dataset are first resized to $299\times299$ images and normalized. Then, each image is padded by $15$ pixels on each side, after which a random crop of size $299\times299$ is extracted for data augmentation. The test and out-of-distribution data are resized to $299\times299$ as well and then normalized by the same normalization parameters as training data.

To our knowledge, the ensemble baseline of \cite{lakshminarayanan2017simple} is the only technique that can scale up to ImageNet for out-of-distribution detection. To show that fitted ensembles are also scalable to large-scale real-world problems we trained a fitted ensemble consisting of three member-classifiers on ImageNet-2012. The first two constraint-problems are formed by merging every two consecutive classes from the total $1$,$000$ classes of ImageNet. That is, the process merges the first and second classes, third and fourth classes, etc., to form the first superclass space, and then merges second and third classes, fourth and fifth classes, etc to construct the second superclass space. Note that these two superclass spaces define a sufficient set of constraints to make a classifier over the entire $1$,$000$ classes (one sequel). In this section, we call such a combined classifier, $classifier 1$.

The last sequel is simply the original dataset with the same $1$,$000$ classes, and, for simplicity, we refer to its corresponding member-classifier as $classifier 2$. All the member-classifiers are neural networks with DenseNet-BC \cite{huang2017densely} architecture with growth rate of $64$, $16$ initial features (with no additional layers before entering the first block) and block configurations of sizes $6$, $6$, $12$, $12$, respectively. 

Each network is trained for $40$ epochs, where in the first $20$ epoch the learning rate is increased by a factor of $1.25$ after each epoch and is decayed by a factor of $0.8$ after each epoch during the rest of the training. The initial learning rate is $0.001$.

The optimizer is SGD with Nesterov momentum, batch size of $64$, the momentum of $0.9$, and weight decay of $10^{-4}$. 
In the interest of time, we did not optimize the architecture or training procedure because the goal of this experiment is not to improve classification accuracy but rather the consistency of predictive confidence values.

\subsubsection{Results}
To confirm that fitted ensembles can adjust predictive confidences more accurately we have calculated the average predictive confidence (maximum activation) for all correctly classified and misclassified examples for $classifier 1$, $classifier 2$, and the whole fitted ensemble, along with the classification accuracy of each. Furthermore, to compare the performance of the fitted ensemble with a conventional ensemble we trained a regular ensemble consisting of two models (comparable to two sequels) following the same architecture, data preprocessing/augmentation, and optimization procedure as the networks of the fitted ensemble.   

Table~\ref{pcc-table} summarizes the confidence statistics of $classifier 1$, $classifier 2$, the overall fitted ensemble (denoted by f-ensemble in the table), and the conventional ensemble. The overall result is more consistent predictive confident values in the fitted ensemble, where the average predictive confidence for misclassified examples is about $23\%$ better (lower) compared to the conventional ensemble, while the average confidence of the correctly classified examples is about $11\%$ lower. 

\begin{table*}[t]
\caption{Average predictive confidence in ImageNet experiment. Better predictive confidence consistency is seen in fitted model through its lower predictive confidence of misclassified examples while improving accuracy of member classifiers.}
\label{pcc-table}
\vskip 0.15in
\begin{center}
\begin{small}
\begin{sc}
\begin{tabular}{lcccr}
\hline
 metric & classifier 1 & classifier 2 & f-ensembles & ensemble  \\
\hline
 Avg. miss-prediction conf.& 38.79(30.10)& 57.83(24.26)&29.95(28.54)&52.94(23.28)\\
Avg. correct prediction conf.& 80.58(27.43)& 88.09(17.94)&74.66(31.11)&85.80(19.22)\\
Avg. total prediction conf.& 70.11& 80.33&64.14&78.03\\
Classification Accuracy& 74.93& 74.36&76.49&76.34\\
\hline
\end{tabular}
\end{sc}
\end{small}
\end{center}
\vskip -0.1in
\end{table*}

In the next experiment, the performance of this fitted ensemble against out-of-distribution data is evaluated. We apply four out-of-distribution datasets. The first two are artificial images generated by sampling from Rademacher and isotropic Gaussian distributions, respectively. The third dataset is Textures\footnote{Available at http://www.robots.ox.ac.uk/~vgg/data/dtd/.}, which consists of $5$,$640$ images of $47$ different type of textures, and the last dataset is SVHN. We use these datasets because they do not include any objects that resemble those in ImageNet.

Table~\ref{ood-table} shows the performance of each model in handling out-of-distribution data. There are three metrics applied in this experiment. The first metric, called FPR-at-$95\%$-TPR \cite{devries2018learning,hendrycks2018deep}, reports a false positive rate with the threshold that results in a $95\%$ true positive rate. Here the in-distribution data are positive and out-of-distribution data are considered negative examples. Intuitively by applying the threshold value on confidence values that $95\%$ of in-distribution examples pass, this metric measures the rate of out-of-distribution examples that are more confident than the least $5\%$ in-distribution examples. 

The second metric is the area under the ROC curve (AUROC), the curve that plots true positive rate against false positive rate for all different threshold values. This holistic metric can be interpreted as the probability that the classifier is more confident about an in-distribution example than an out-of-distribution example. Therefore, the higher AUROC the better. An AUROC of $100\%$ shows that the classifier always detects out-of-distribution data, regardless of the threshold value, and an AUROC of $50\%$ is no better than a random classifier.

The last metric, called detection error, reports the highest possible detection rate if we choose the optimal threshold value \cite{liu2016delving}. Note that these metrics can be sensitive to the relative size of in-distribution versus out-of-distribution sets. For the two artificial datasets, we used the same amount of in- and out-of-distribution data, while for the Textures and SVHN datasets we included all images and all test images, respectively.

The results suggest that the fitted ensemble is overall superior to both of its component classifiers and to the regular ensemble in detecting out-of-distributed examples. On both artificial datasets, the fitted ensemble performs significantly better than the conventional ensemble on both FPR at $95\%$ TPR and AUROC metrics. 

The Texture images are also detected very significantly better by the fitted ensemble than the others. However, the SVHN images are detected significantly better by $classifier 1$, and combining $classifier 1$ by $classifier 2$ slightly degrades the performance of the fitted ensemble. Even so, the fitted ensemble beats the conventional ensemble with a large margin (about $45\%$ improvement by AUROC measure).

Interestingly, the AUROC value for the conventional ensemble, and that of $classifier 2$, are both below $50\%$, implying they both perform worse than a random classifier for detecting SVHN images. In contrast, the high detection rate of SVHN images by $classifier 1$, as evidenced by AUROC, reaffirms that the reformulation of classification problems in fitted ensembles can indeed lead to a novel type of diversity in ensemble techniques, a kind of diversity that is not attainable by regular ensembling.

\begin{table}[t]
\caption{Performance of different models towards out-of-distribution data.}
\label{ood-table}
\vskip 0.15in
\begin{center}
\begin{small}
\begin{sc}
\begin{tabular}{lcccr}
\hline
 ds\& metric & class. 1 & class. 2 & f-ens. & ens.  \\
\hline
FPR at 95\%\\
TPR\\
\hline
Rademacher & 100& 100&49.89&98.71\\
Gaussian & 99.33&100&41.84&99.00 \\
Textures & 74.06&81.88&70.60&80.74 \\
SVHN &  78.53&96.53&84.28&96.75\\
\hline
Area under\\
ROC curve\\
\hline
Rademacher & 77.63& 67.06&94.51&89.90\\
Gaussian & 87.84&66.70&94.89&89.79 \\
Textures & 79.16&72.45&80.84&76.65 \\
SVHN &87.41&40.24&82.97&37.74\\
\hline
Best detection\\
error\\
\hline
Rademacher & 17.26& 22.62&6.90&7.50\\
Gaussian & 11.64&22.65&6.93&7.54 \\
Textures & 10.08&10.14&9.98&10.14 \\
SVHN &18.88&34.23&23.29&34.24\\
\hline
\end{tabular}
\end{sc}
\end{small}
\end{center}
\vskip -0.1in
\end{table}

\subsubsection{Dissimilar-architecture Ensemble Experiment}

In the previous experiment, the conventional ensemble consisted of two CNNs with an identical architecture (similar to network components of the fitted ensemble). In this section, we evaluate the effect of including dissimilar network architectures in an ensemble on improving predictive confidence consistency. To our knowledge, this approach has not been investigated for this problem before. Nevertheless, it is interesting to compare the kind of diversity introduced to an ensemble by a fitted ensemble with the diversity induced by dissimilar network architectures in an ensemble.

To this end, we construct an ensemble by combining two pretrained models on ImageNet, namely DenseNet-121 and DenseNet-161, from the PyTorch framework, and compared them with the fitted ensemble from the previous experiments. Note that the two pretrained models are significantly deeper and more accurate than $classifier 1$ and $classifier 2$ of the fitted ensemble.

The confidence statistics of the fitted ensemble (f-ensemble), and the ensemble of the two pretrained models from PyTorch, are summarized in Table~\ref{pcc2-table}. Interestingly, the result shows more consistent predictive confident values in the fitted ensemble, where the average predictive confidence for misclassified examples is about $10\%$ better (lower) compared to the regular ensemble, while the average confidence of the correctly classified examples is only $1\%$ lower. 

\begin{table}[t]
\caption{Average predictive confidence of fitted ensemble versus regular ensemble. The fitted ensemble makes more consistent predictive confidence value evident by its lower predictive confidence of misclassified examples.}
\label{pcc2-table}
\vskip 0.15in
\begin{center}
\begin{small}
\begin{sc}
\begin{tabular}{lcccr}
\hline
 metric & f-ensembles & ensemble  \\
\hline
Avg. miss-prediction\\ conf.&29.95(28.54)&40.33(22.78)\\
Avg. correct \\prediction conf.& 74.66(31.11)&75.71(23.90)\\
Avg. total \\prediction conf.&64.14&68.11\\
Classification \\Accuracy& 76.49&78.53\\
\hline
\end{tabular}
\end{sc}
\end{small}
\end{center}
\vskip -0.1in
\end{table}

Following the previous section, the next experiment evaluates the performance of the two ensembles while encountering out-of-distribution data. The experimental setup is the same as the out-of-distribution experiment in the previous section.

Table~\ref{ood2-table} depicts the performance of the two ensembles in handling out-of-distribution data. Again the FPR at $95\%$ TPR, area under the ROC curve, and detection error are applied to quantify the detection performance. 

The results suggest that the fitted ensemble is performing better on three of the four benchmarks, namely $Rademacher$, $Gaussian$, and the Textures datasets. However, the dissimilar-architecture ensemble can detect the SVHN images far better than the fitted ensemble. These results suggest that applying different network architectures in an ensemble can be more advantageous than the common similar-architecture practice for detecting out-of-distribution data. Nevertheless, fitted ensembles can enjoy the same kind of diversity to perform even better.

\begin{table}[t]
\caption{Performance of fitted versus regular ensemble towards out-of-distribution data.}
\label{ood2-table}
\vskip 0.15in
\begin{center}
\begin{small}
\begin{sc}
\begin{tabular}{lcccr}
\hline
 dataset\& metric & f-ensemble & ensemble  \\
\hline
FPR at 95\% TPR\\
\hline
Rademacher &49.89&99.64\\
Gaussian & 41.84&98.60 \\
Textures & 70.60&78.32 \\
SVHN &84.28&15.58\\
\hline
Area under ROC curve\\
\hline
Rademacher & 94.51&87.37\\
Gaussian & 94.89&89.75 \\
Textures & 80.84&76.51 \\
SVHN &82.97&97.27\\
\hline
Best detection error\\
\hline
Rademacher & 6.90&12.26\\
Gaussian & 6.93&10.27 \\
Textures & 9.98&10.10 \\
SVHN &23.28&8.13\\
\hline
\end{tabular}
\end{sc}
\end{small}
\end{center}
\vskip -0.1in
\end{table}

\section{Discussion and Conclusions}
This study addresses the confidence inconsistencies that can occur in modern neural network models. We proposed a new framework, called separable concept learning, that generalizes the rigid set up of the conventional classification framework to a more flexible setting that considers out-of-distribution examples. This new framework evaluates classifiers' performance by simulating a system of classifiers that work in harmony to achieve a unified goal, while each classifier may encounter examples of new but relevant classes.

Furthermore, to handle the inconsistent prediction confidence in neural networks we introduced fitted ensembles. Fitted ensembles introduce a new dimension to boost ensembles, where the diversity is preserved by learning different superclasses as output units. Our experiments show that this new reformulation of the original problem can enable fitted ensembles to produce less overconfident predictions and detect even more out-of-distribution examples compared to regular ensembles. Furthermore, our ImageNet experiments with fitted ensembles show scalability of this approach.

Overall, a new method is introduced to rectify neural network predictions and make them more suitable for handling unknown situations, and a new paradigm for evaluation is introduced to reveal their advantages.  The hope is that these contributions can lead to better techniques to eventually make reliable classifiers and pave the way to a more consistent evaluation of them.

\bibliographystyle{IEEEtran}
\bibliography{IEEEabrv,mRef}

\begin{thebibliography}{10}
\providecommand{\url}[1]{#1}
\csname url@samestyle\endcsname
\providecommand{\newblock}{\relax}
\providecommand{\bibinfo}[2]{#2}
\providecommand{\BIBentrySTDinterwordspacing}{\spaceskip=0pt\relax}
\providecommand{\BIBentryALTinterwordstretchfactor}{4}
\providecommand{\BIBentryALTinterwordspacing}{\spaceskip=\fontdimen2\font plus
\BIBentryALTinterwordstretchfactor\fontdimen3\font minus
  \fontdimen4\font\relax}
\providecommand{\BIBforeignlanguage}[2]{{%
\expandafter\ifx\csname l@#1\endcsname\relax
\typeout{** WARNING: IEEEtran.bst: No hyphenation pattern has been}%
\typeout{** loaded for the language `#1'. Using the pattern for}%
\typeout{** the default language instead.}%
\else
\language=\csname l@#1\endcsname
\fi
#2}}
\providecommand{\BIBdecl}{\relax}
\BIBdecl

\bibitem{breiman2017classification}
L.~Breiman, \emph{Classification and regression trees}.\hskip 1em plus 0.5em
  minus 0.4em\relax Routledge, 2017.

\bibitem{DudaHart2nd}
R.~O. Duda, P.~E. Hart, and D.~G. Stork, \emph{Pattern Classification},
  2nd~ed.\hskip 1em plus 0.5em minus 0.4em\relax John Wiley and Sons, 2000.

\bibitem{dietterich1998approximate}
T.~G. Dietterich, ``Approximate statistical tests for comparing supervised
  classification learning algorithms,'' \emph{Neural computation}, vol.~10,
  no.~7, pp. 1895--1923, 1998.

\bibitem{duda1973pattern}
R.~O. Duda, P.~E. Hart, and D.~G. Stork, \emph{Pattern classification and scene
  analysis}.\hskip 1em plus 0.5em minus 0.4em\relax Wiley New York, 1973,
  vol.~3.

\bibitem{cover1967nearest}
T.~M. Cover, P.~E. Hart \emph{et~al.}, ``Nearest neighbor pattern
  classification,'' \emph{IEEE transactions on information theory}, 1967.

\bibitem{fisher1936}
R.~A. Fisher, ``The use of multiple measurements in taxonomic problems,''
  \emph{Annals of eugenics}, vol.~7, no.~2, pp. 179--188, 1936.

\bibitem{hastie}
T.~Hastie, R.~Tibshirani, and J.~Friedman, \emph{The Elements of Statistical
  Learning}, ser. Springer Series in Statistics.\hskip 1em plus 0.5em minus
  0.4em\relax New York, NY, USA: Springer New York Inc., 2001.

\bibitem{fisher1938}
R.~A. Fisher, ``The statistical utilization of multiple measurements,''
  \emph{Annals of eugenics}, vol.~8, no.~4, pp. 376--386, 1938.

\bibitem{vapnik}
V.~Vapnik, \emph{The nature of statistical learning theory}.\hskip 1em plus
  0.5em minus 0.4em\relax Springer Science \& Business Media, 2013.

\bibitem{guo2017calibration}
C.~Guo, G.~Pleiss, Y.~Sun, and K.~Q. Weinberger, ``On calibration of modern
  neural networks,'' in \emph{Proceedings of the 34th International Conference
  on Machine Learning-Volume 70}, 2017, pp. 1321--1330.

\bibitem{MM2020}
J.~Zhang, B.~Kailkhura, and T.~Y. Han, ``Mix-n-match : Ensemble and
  compositional methods for uncertainty calibration in deep learning,'' in
  \emph{Proceedings of the 37th International Conference on Machine Learning,
  {ICML}}, ser. Proceedings of Machine Learning Research, 2020.

\bibitem{lecun1989backpropagation}
Y.~LeCun, B.~Boser, J.~S. Denker, D.~Henderson, R.~E. Howard, W.~Hubbard, and
  L.~D. Jackel, ``Backpropagation applied to handwritten zip code
  recognition,'' \emph{Neural computation}, 1989.

\bibitem{our}
N.~Kardan and K.~O. Stanley, ``Mitigating fooling with competitive overcomplete
  output layer neural networks,'' in \emph{Neural Networks (IJCNN), 2017
  International Joint Conference on}.\hskip 1em plus 0.5em minus 0.4em\relax
  IEEE, 2017, pp. 518--525.

\bibitem{hendrycks2016baseline}
D.~Hendrycks and K.~Gimpel, ``A baseline for detecting misclassified and
  out-of-distribution examples in neural networks,'' \emph{arXiv preprint
  arXiv:1610.02136}, 2016.

\bibitem{liang2017enhancing}
S.~Liang, Y.~Li, and R.~Srikant, ``Enhancing the reliability of
  out-of-distribution image detection in neural networks,'' \emph{arXiv
  preprint arXiv:1706.02690}, 2017.

\bibitem{lee2018simple}
K.~Lee, K.~Lee, H.~Lee, and J.~Shin, ``A simple unified framework for detecting
  out-of-distribution samples and adversarial attacks,'' in \emph{Advances in
  Neural Information Processing Systems}, 2018, pp. 7167--7177.

\bibitem{devries2018learning}
T.~DeVries and G.~W. Taylor, ``Learning confidence for out-of-distribution
  detection in neural networks,'' \emph{arXiv preprint arXiv:1802.04865}, 2018.

\bibitem{hendrycks2018deep}
D.~Hendrycks, M.~Mazeika, and T.~G. Dietterich, ``Deep anomaly detection with
  outlier exposure,'' \emph{arXiv preprint arXiv:1812.04606}, 2018.

\bibitem{ApoorvECCV}
A.~Vyas, N.~Jammalamadaka, X.~Zhu, D.~Das, B.~Kaul, and T.~L. Willke,
  ``Out-of-distribution detection using an ensemble of self supervised
  leave-out classifiers,'' in \emph{Computer Vision - {ECCV} 2018 - 15th
  European Conference}, 2018.

\bibitem{sastry2020detecting}
C.~S. Sastry and S.~Oore, ``Detecting out-of-distribution examples with gram
  matrices,'' in \emph{International Conference on Machine Learning}.\hskip 1em
  plus 0.5em minus 0.4em\relax PMLR, 2020, pp. 8491--8501.

\bibitem{chen2020robust}
J.~Chen, X.~Wu, Y.~Liang, S.~Jha \emph{et~al.}, ``Robust out-of-distribution
  detection in neural networks,'' \emph{arXiv preprint arXiv:2003.09711}, 2020.

\bibitem{Energy2020}
W.~Liu, X.~Wang, J.~D. Owens, and Y.~Li, ``Energy-based out-of-distribution
  detection,'' in \emph{Annual Conference on Neural Information Processing
  Systems, NeurIPS 2020}.

\bibitem{galuncertainty}
Y.~Gal, ``Uncertainty in deep learning,'' Ph.D. dissertation, PhD thesis,
  University of Cambridge, 2016.

\bibitem{Ovadia}
J.~Snoek, Y.~Ovadia, E.~Fertig, B.~Lakshminarayanan, S.~Nowozin, D.~Sculley,
  J.~V. Dillon, J.~Ren, and Z.~Nado, ``Can you trust your model's uncertainty?
  evaluating predictive uncertainty under dataset shift,'' in \emph{Annual
  Conference on Neural Information Processing Systems, NeurIPS 2019}.

\bibitem{posterior2020}
B.~Charpentier, D.~Z{\"{u}}gner, and S.~G{\"{u}}nnemann, ``Posterior network:
  Uncertainty estimation without {OOD} samples via density-based
  pseudo-counts,'' in \emph{Annual Conference on Neural Information Processing
  Systems 2020, NeurIPS 2020}.

\bibitem{Ustatistics2020}
J.~Schupbach, J.~W. Sheppard, and T.~Forrester, ``Quantifying uncertainty in
  neural network ensembles using u-statistics,'' in \emph{2020 International
  Joint Conference on Neural Networks, {IJCNN}}, 2020.

\bibitem{zhang2017universum}
X.~Zhang and Y.~LeCun, ``Universum prescription: Regularization using unlabeled
  data,'' in \emph{Thirty-First AAAI Conference on Artificial Intelligence},
  2017.

\bibitem{lakshminarayanan2017simple}
B.~Lakshminarayanan, A.~Pritzel, and C.~Blundell, ``Simple and scalable
  predictive uncertainty estimation using deep ensembles,'' in \emph{Advances
  in Neural Information Processing Systems}, 2017, pp. 6402--6413.

\bibitem{wolpert1997no}
D.~H. Wolpert, W.~G. Macready \emph{et~al.}, ``No free lunch theorems for
  optimization,'' \emph{IEEE transactions on evolutionary computation}, 1997.

\bibitem{ourarxiv}
N.~Kardan and K.~O. Stanley, ``Fitted learning: Models with awareness of their
  limits,'' \emph{arXiv preprint arXiv:1609.02226}, 2016.

\bibitem{bishop2006pattern}
C.~M. Bishop, \emph{Pattern recognition and machine learning}.\hskip 1em plus
  0.5em minus 0.4em\relax springer, 2006.

\bibitem{simonyan2014very}
K.~Simonyan and A.~Zisserman, ``Very deep convolutional networks for
  large-scale image recognition,'' \emph{arXiv preprint arXiv:1409.1556}, 2014.

\bibitem{huang2017densely}
G.~Huang, Z.~Liu, L.~Van Der~Maaten, and K.~Q. Weinberger, ``Densely connected
  convolutional networks,'' in \emph{Proceedings of the IEEE conference on
  computer vision and pattern recognition}, 2017, pp. 4700--4708.

\bibitem{he2015delving}
K.~He, X.~Zhang, S.~Ren, and J.~Sun, ``Delving deep into rectifiers: Surpassing
  human-level performance on imagenet classification,'' in \emph{Proceedings of
  the IEEE international conference on computer vision}, 2015, pp. 1026--1034.

\bibitem{glorot2010understanding}
X.~Glorot and Y.~Bengio, ``Understanding the difficulty of training deep
  feedforward neural networks,'' in \emph{Proceedings of the thirteenth
  international conference on artificial intelligence and statistics}, 2010,
  pp. 249--256.

\bibitem{kingma2014adam}
D.~P. Kingma and J.~Ba, ``Adam: A method for stochastic optimization,''
  \emph{arXiv preprint arXiv:1412.6980}, 2014.

\bibitem{liu2016delving}
Y.~Liu, X.~Chen, C.~Liu, and D.~Song, ``Delving into transferable adversarial
  examples and black-box attacks,'' \emph{arXiv preprint arXiv:1611.02770},
  2016.

\end{thebibliography}

\end{document}